\documentclass{article}

     \usepackage[compress]{cite}

\usepackage[utf8]{inputenc} 
\usepackage[T1]{fontenc}    
\usepackage{lmodern}

\usepackage{hyperref}       
\usepackage{url}            
\usepackage{booktabs}       
\usepackage{amsmath, amssymb, scalerel, amsthm, relsize, xcolor}
\usepackage{nicefrac}       
\usepackage{microtype}      
\usepackage{wrapfig}
\usepackage[small]{caption}
\usepackage{subcaption}
\usepackage[left=1in,right=1in,top=1in,bottom=1in]{geometry}

\newtheorem{lemma}{Lemma}
\newtheorem{theorem}{Theorem}
\newtheorem{example}{Example}

\title{Quantizing Multiple Sources to a Common Cluster Center: An Asymptotic Analysis}

\author{%
  Erdem Koyuncu\thanks{This work was supported in part by the NSF Award CCF--1814717 and in part by an award from the University of Illinois at Chicago Discovery Partners Institute Seed Funding Program. } \\
  Department of Electrical and Computer Engineering\\
  University of Illinois at Chicago\\
  \texttt{ekoyuncu@uic.edu} \\
}

\begin{document}



\maketitle

\begin{abstract}

We consider quantizing an $Ld$-dimensional sample, which is obtained by concatenating $L$ vectors from datasets of $d$-dimensional vectors, to a $d$-dimensional cluster center. The distortion measure is the weighted sum of $r$th powers of the distances between the cluster center and the samples. For $L=1$, one recovers the ordinary center based clustering formulation. The general case $L>1$ appears when one wishes to cluster a dataset through $L$ noisy observations of each of its members. We find a formula for the average distortion performance in the asymptotic regime where the number of cluster centers are large. We also provide an algorithm to numerically optimize the cluster centers and verify our analytical results on real and artificial datasets. In terms of faithfulness to the original (noiseless) dataset, our clustering approach outperforms the naive approach that relies on quantizing the $Ld$-dimensional noisy observation vectors to $Ld$-dimensional centers.
\end{abstract}


%
%
%
%
%

\section{Introduction}
\subsection{Ordinary Center Based Clustering}
The goal of clustering is to partition an unlabeled dataset into disjoint sets or clusters such that each cluster only consists of similar members of the dataset \cite{jain1999data, jain2010data, gan2007data}. Of particular interest to this work are center or centroid based clustering methods, as described in the following. Let $D = \{y_{1,i} \}_{i=1}^m \subset \mathbb{R}^d$ be a dataset  of $d$-dimensional vectors, whose elements are drawn according to a random vector $X_1$. In classical $k$-means clustering \cite{macqueen1967some, steinhauskmeans, lloyd1982least}, one is interested in finding the optimal cluster centers $u_1,\ldots,u_n$ that minimize the average distortion
\begin{align}
\label{kmeansproblem}
(u_1,\ldots,u_n) \mapsto \frac{1}{m}\sum_{i=1}^m \min_k \|u_k - y_{1,i}\|^2.
\end{align}
Distinct clusters can then be identified via the Voronoi cells $\{y \in D:\|y - u_i\| \leq \|y - u_j\|,\,\forall j\},\,i=1,\ldots,n$ (ties are broken arbitrarily). Several variations to the basic formulation in (\ref{kmeansproblem}) have been studied. For example, the squared error distortion measure $(u,y) \mapsto \|u-y\|^2$ between the cluster center $u_k$ and the dataset sample $y_{1,i}$ in (\ref{kmeansproblem}) can be replaced with the $r$th power distortion measure $(u,y) \mapsto \|u-y\|^r$ \cite{bucklew1982multidimensional}, a quadratic distortion measure $(u,y) \mapsto (u-y)^T A (u-y)$, where $A$ is a positive semi-definite matrix \cite{466658,li1999asymptotic}, or a Bregman divergence \cite{banerjee2005clustering,fischer2010quantization}. 


Finding a (globally-)optimal solution to (\ref{kmeansproblem}) is known to be an NP-hard problem \cite{aloise2009np}. Nevertheless, locally-optimal solutions can be found using the $k$-means algorithm or its extensions such as the generalized Lloyd algorithm \cite{lloyd1982least, linde1980algorithm}. Moreover, vector quantization theory \cite{720541, gersho2012vector} provides a precise description of the structure of optimal solutions and the corresponding minimum average distortions in the asymptotic regimes $n,m \rightarrow\infty$ \cite{1056490, liu2016clustering}.

\subsection{Problem Statement}

In this paper, we will study the following generalization of (\ref{kmeansproblem}): Consider the dataset 
\begin{align}
\label{noisydataset}
D = \{[y_{1,i} \cdots y_{L,i}]\}_{i=1}^m \subset \mathbb{R}^{d\times L}
\end{align}
containing $m$ observations of the $\mathbb{R}^{d\times L}$-valued random matrix $[X_1 \cdots X_L]$. We wish to minimize
\begin{align}
\label{ourproblem}
(u_1,\ldots,u_n) \mapsto \frac{1}{m}\sum_{i=1}^{m} \min_k \sum_{\ell = 1}^L \lambda_{\ell} \|u_k - y_{\ell,i}\|^r,
\end{align}
where $\lambda_{\ell} > 0$ are some weights. For $L=1$ and $r=2$, we recover the original $k$-means problem (\ref{kmeansproblem}). We provide a complete $n,m\rightarrow\infty$ asymptotic analysis of the minimizers of (\ref{ourproblem}) and the corresponding loss values for the special cases $r=2$ for any $L$, and for any $r \geq 1$ and $L=2$. 

\subsection{Application Scenarios}
\label{secAppScenarios}
An example scenario where the cost function (\ref{ourproblem}) becomes relevant is as follows: Consider a physical process that generates a dataset $D' = \{y_i'\}_{i=1}^m$, and suppose that our goal is cluster $D'$. In practice, we may only access a noisy version of $D'$ through, for example, sensor measurements. For any given sample index $i\in\{1,\ldots,m\}$, given that there are $L$ sensors in total, Sensor $\ell$ can provide a noisy version $y_{\ell,i}$ of the true data sample $y_i'$. In such a scenario, one only has the noisy dataset $D$ in (\ref{noisydataset}) available and cannot access the true dataset $D'$. We thus wish to find a clustering of $D$ that is as faithful to the clustering of $D'$ as possible. The minimization of (\ref{ourproblem}) provides a possible solution to this problem. In fact, (\ref{ourproblem}) exploits the information that each measurement is a noisy version of a common true data sample. This is why a common cluster center is used for all $L$ noisy versions of the data sample. The weights $\lambda_{\ell}$ may be used to control the relative importance of individual sensor measurements: Sensors known to be less noisy may be assigned a larger weight. 

It is also worth mentioning that a straightforward alternative approach for approximating $D'$ from $D$ may be to directly apply the $k$-means algorithm to $D'$. In this case, the problem can be considered to be a special case of multi-modal or multi-view clustering  \cite{chao2017survey, bickel2004multi}, where each noisy measurement corresponds to one individual view. We numerically demonstrate that the approximation performance with this approach, measured in terms of the adjusted Rand index (ARI) \cite{rand1971objective} or the adjusted mutual information (AMI) \cite{vinh2010information}, is inferior to our approach that relies on minimizing (\ref{ourproblem}). The optimal way to approximate a clustering through its noisy measurements warrants a separate investigation and will be left as future work. We refer to reader to \cite{farvardin1990study, dave1991characterization} for other problem formulations that involve clustering noisy data. 

We would also like to note that the objective function (\ref{ourproblem}) can also be interpreted in the context of facility location optimization, at least for the special case $L=2$ (and $L=1$). In fact, many facility location optimization problems can be formulated as clustering or quantization problems \cite{okabe1997locational, farahani2010multiple}. For our scenario, consider packages at locations $y_{1,1},\ldots,y_{1,m}$, which are to be processed at one of the facilities $u_1,\ldots,u_n$, and then conveyed to their destinations $y_{2,1},\ldots,y_{2,m}$, respectively. The cost of conveying a package from one location to another can be modeled to be proportional to the $r$th power of the distance between the two locations \cite{meira2017clustering, meiraconfpapap, czumaj20131, koyuncu2018power}. Thus, the total cost of conveying the $i$th package through the facility at $u_k$ is given by $\|y_{1,i} - u_k\|^r + \|u_k- y_{2,i}\|^r$. The minimum average cost of conveying all packages is then given by (\ref{ourproblem}) for the special case $L = 2$ with $\lambda_{\ell} = 1,\forall \ell$. Minimizing (\ref{ourproblem}) corresponds to optimizing the facility locations. Other applications of vector quantization to sensor or facility location optimization can be found in \cite{ekc17, ekc18}. 
\subsection{Organization of the Paper}
The rest of this paper is organized as follows: In Section \ref{secprelims}, we introduce some well-known results from quantization theory. In Sections \ref{secgeneralratzeroaltxx} and \ref{secgeneralratzeroalt}, we analyze the cases of squared-error distortions and arbitrary powers of errors, respectively. In Section \ref{secNumerical}, we provide numerical results over real and artificial datasets. Finally, in Section \ref{secConclusions}, we draw our main conclusions. Proofs of certain theoretical claims and more numerical experiments can be found in the extended version of the paper.

\newcommand{\pgt}{P_{\scaleto{\mathrm{GT}}{3.5pt}}}
\newcommand{\pgtd}{P_{\scaleto{\mathrm{GT}}{3.5pt},{\scaleto{\mathrm{D}}{3.5pt}}}}
\newcommand{\puav}{P_{\scaleto{\mathrm{UAV}}{3.5pt}}}
\newcommand{\pdist}{P_{\scaleto{\mathrm{D}}{3.5pt}}}

\section{Preliminaries}
\label{secprelims}
Throughout the paper, we will present our analytical results for the case of having $m\rightarrow\infty$ observations from the dataset $D$. In particular, for the simple $k$-means scenario in (\ref{kmeansproblem}), letting $m\rightarrow\infty$ amounts to replacing the empirical sum with the integral
\begin{align}
\label{qowieoqwieq} \delta_1(U) \triangleq \int \min_k \|u_k - x\|^2 f_{X_1}(x)\mathrm{d}x,
\end{align}
where $U = (u_1,\ldots,u_n)$ is the quantizer codebook, and $f_{X_1}$ represents the probability density function (PDF) of $X_1$.  We omit the domain of integration when it is clear from the context. 

For $n=1$, it is easily shown that the minimizer of (\ref{qowieoqwieq}) is the  centroid $u_1 = \mathrm{E}[X_1]$. On the other hand, finding the  minimizers of (\ref{qowieoqwieq}) is a hopeless problem for a general number of centers $n > 1$. On the other hand, vector quantization theory yields a precise $n\rightarrow\infty$ characterization. In fact, we have  the asymptotic  result \cite{bennett1948spectra, 1056490}
\begin{align}
\label{asdasdasdasd}
 \min_U  \delta_1(U) = \kappa_{d}n^{-\frac{2}{d}} \|f_{X_1}\|_{\frac{d}{d+2}} + o(n^{-\frac{2}{d}}),
\end{align}
where $\kappa_d$ is a constant that depends only on the dimension $d$, and 
\begin{align}
\|f_{X_1}\|_p & \triangleq \left(\int (f_{X_1}(x))^p \mathrm{d}x\right)^{\frac{1}{p}}
\end{align}
is the $p$-norm of the density $f_{X_1}$.\ The sequence of quantizer codebooks that achieve the performance in (\ref{asdasdasdasd}) has the following property: There exists a continuous function $\lambda(x)$ such that at the cube $[x,x+\mathrm{d}x]$, optimal quantizer codebooks contain $n\lambda(x)\mathrm{d}x$  points as $n\rightarrow\infty$. Hence, $\lambda(x)$ can be thought as a  ``point density function'' and obeys the normalization $\int \lambda(x)\mathrm{d}x = 1$. For the squared-error distortion, the optimal point density function depends on the input distribution through
\begin{align}
\label{lambdadef}
\lambda(x) = \frac{f_{X_1}^{\frac{d}{d+2}}(x)}{\int f_{X_1}^{\frac{d}{d+2}}(y)\mathrm{d}y}.
\end{align} 
Equivalently, we say that $\lambda$ is proportional to $\smash{f_{X_1}^\frac{d}{d+2}}$. The question is now how the $n\lambda(x)\mathrm{d}x$ quantization points are to be deployed optimally inside each cube $[x,x+\mathrm{d}x]$. Since the underlying density $f_{X_1}$ is approximately uniform on $[x,x+\mathrm{d}x]$, the question is equivalent to finding the structure of an optimal quantizer for a uniform distribution. For one and two dimensions, the optimal quantizers are known to be the uniform and the hexagonal lattice quantizers, respectively (Thus, the $n\lambda(x)\mathrm{d}x$ points should follow a hexagonal lattice on the square  $[x,x+\mathrm{d}x]$ in an optimal quantizer). We thus have $\smash[b]{\kappa_1 = \frac{1}{12}}$ and $\smash{\kappa_2 = \frac{5}{18\sqrt{3}}}$, corresponding to the normalized second moment of the interval and the regular hexagon respectively. For a set $A \subset \mathbb{R}^d$ with $\int_A x\mathrm{d}x = 0$, the normalized second moment is defined as 
\begin{align}
\kappa(A) \triangleq \frac{\int_A \|x\|^2 \mathrm{d}x}{(\int_A \mathrm{d}x)^{\frac{d+2}{d}}}.
\end{align}
For $d\geq 3$, the optimal quantizer structure and $\kappa_d$ has remained unknown. We conclude by noting that the arguments involving infinitesimals and point density functions can be formalized; see \cite{graf2007foundations}.

\section{Squared-Error Distortions}
\label{secgeneralratzeroaltxx}
Let us now consider our problem of quantizing multiple sources to a common cluster center, as formulated in (\ref{ourproblem}). Focusing on the case of having $m\rightarrow\infty$ observations from the dataset $D$, we replace the empirical sum in (\ref{ourproblem}) with the integral
\begin{align}
\label{objfun}
\delta(U) \triangleq \int \min_k \sum_{i = 1}^L \lambda_{i} \|u_k - x_{i}\|^r f_{\mathbf{X}}(\mathbf{x})\mathrm{d}\mathbf{x},
\end{align}
where $\mathbf{x} = [x_1 \cdots x_L]$ represents a realization of the random matrix $\mathbf{X} = [X_1 \cdots X_L]$, and $f_{\mathbf{X}}(\mathbf{x})$ is the PDF of $\mathbf{X}$. 



We first consider the case of squared-error distortions $r=2$, which allow a simple characterization of the optimal average distortions. In fact, for $r=2$, the integrand in (\ref{objfun}) can be rewritten as
\begin{align}
\label{sdjlkajdsa}
\sum_{i = 1}^L \lambda_{i} \|u_k - x_{i}\|^2 & =c_1\Biggl\|u_k -   \frac{1}{c_1}\sum_i \lambda_i x_i \Biggr\|^2 + \sum_i \lambda_i \|x_i\|^2 -  \frac{1}{c_1}\Biggl\|  \sum_i \lambda_i x_i \Biggr\|^{\!2} ,
%
%
\end{align}
where $c_1 \triangleq \sum_i \lambda_i$. The equivalence (\ref{sdjlkajdsa}) can easily be verified by  expanding the squared Euclidean norms on both sides via the formula $\|\alpha + \beta\|^2 = \|\alpha\|^2 + \|\beta\|^2 + 2\alpha^T \beta$, where $\alpha$ and $\beta$ are arbitrary vectors. Let us now define a new random variable 
\begin{align}
Z \triangleq \frac{1}{c_1} \sum_i \lambda_i X_i,
\end{align}
and, with regards to the last two terms in (\ref{sdjlkajdsa}), the expected value
\begin{align}
c_2 \triangleq \mathrm{E}\biggl[\sum_i \lambda_i \|X_i\|^2 -  \frac{1}{c_1}\biggl\|  \sum_i \lambda_i X_i \biggr\|^{\!2\,} \biggr].
\end{align}
By substituting  (\ref{sdjlkajdsa}) to (\ref{objfun}), we can then arrive at
\begin{align}
\label{aksdjaldsas}
\delta(U) = c_2 + c_1\int \min_k \|u_k - x\|^2 f_{Z}(z)\mathrm{d}z.
\end{align}
We observe that the integral in (\ref{aksdjaldsas}) is merely the average squared-error distortion of a quantizer given a source with density $\smash{Z}$. 
Therefore, when $r=2$, optimal quantization of multiple sources to a common cluster center is equivalent to the optimal quantization of the single source $\smash{Z}$ with the usual squared-error distortion measure. It follows that the results of Section \ref{secprelims} are directly applicable, and we have the following theorem.

\begin{theorem}
\label{theorem1}
Let $r=2$. As $n\rightarrow\infty$, we have 
\begin{align}
\label{qweqweupoqwe} \min_U \delta(U) = c_2 + c_1 \kappa_{d}n^{-\frac{2}{d}} \|f_{{Z}}\|_{\frac{d}{d+2}} + o(n^{-\frac{2}{d}})
\end{align}
Moreover, the optimal point density function that achieves (\ref{qweqweupoqwe}) is proportional to $\smash{f_{{Z}}^{\frac{d}{d+2}}}$.
\end{theorem}
\begin{proof}
The asymptotic distortion in (\ref{qweqweupoqwe})  follows immediately from (\ref{aksdjaldsas}) and (\ref{lambdadef}). The optimal point density function is as given by (\ref{asdasdasdasd}).
\end{proof}

We have thus precisely characterized the asymptotic average distortion for the case $r=2$. Note that when $L=1$, which corresponds to ordinary quantization with squared-error distortion,  the average distortion decays to zero as the number of quantizer centers $n$ grows to infinity. Theorem \ref{theorem1} demonstrates that when $L > 1$, the average distortion converges to $c_2$, which is in general non-zero. The reason is that when $L > 1$, a single quantizer center is used to reproduce mutiple sources, which makes zero distortion impossible to achieve whenever the sources are not identical.


\section{Distortions with Arbitrary Powers of Errors}
\label{secgeneralratzeroalt}
We now consider the achievable performance for a general $r \neq 2$. We also consider the case $L=2$. Without loss of generality, let $\lambda_1 = 1$. In this case, the objective function in (\ref{objfun}) takes the form
\begin{align}
\label{qpiwoeiqpwoeiqw}
 \delta(U) = \iint \min_i \Bigl\{ \|x_1-u_i\|^r + \lambda_2 \|x_2 - u_i\|^r \Bigr\} f_{X_1,X_2}(x_1,x_2)\mathrm{d}x_1\mathrm{d}x_2.
\end{align}
The main difficulty for $r \neq 2$ is that  and an algebraic manipulation of the form (\ref{sdjlkajdsa}) is not available. Nevertheless, it turns out that an analysis in the high-resolution regime $n\rightarrow\infty$ is still feasible. First, we need the following basic lemma.

\newcommand{\hone}{\xi}

\begin{lemma} 
\label{basiclemma}
Let $\hone(u) = \|x_1 - u\|^r + \lambda_2 \|x_2 - u\|^r$. The global minimizer of $\hone$ is $z = \frac{x_1 + \alpha x_2}{1 + \alpha}$, where $\smash{\alpha \triangleq \lambda_2^{\frac{1}{r-1}}}$. The corresponding global minimum is $\smash{\hone(z) = \frac{\alpha^{r-1}}{(1+\alpha)^{r-1}}\|x_1 - x_2\|^r}$. 
\end{lemma}
\begin{proof}
Note that $\hone$ is convex in $u$ so that is has a global minimum. Observe that this global minimum should be located on the line $l$ that connects $x_1$ and $x_2$. In fact, suppose that the global minimizer $z$ does not belong to $l$. We can project $z$ to $l$ to come up with a new point $z'$ that satisfies $\|x_1 - z' \| < \|x_1 - z\|$ and $\| x_2 - z'\| < \|x_2 - z\|$. This implies $\hone(z') < \hone(z)$ and thus contradicts the optimality of $z$. Given that $z$ should be on $l$, it can be written in the form $z = \frac{x_1 + \alpha x_2}{1 + \alpha}$ for some $\alpha \in \mathbb{R}$. We have
\begin{align}
\label{qowieoiqweqw}
\hone(z) & =  \left\|x_1 -  \frac{x_1 + \alpha x_2}{1 + \alpha}\right\|^r + \lambda_2 \left\|x_2 -  \frac{x_1 + \alpha x_2}{1 + \alpha}\right\|^r  \\
\label{qowieoiqweqw2}
 & =  \frac{\lambda_2 + \alpha^r}{(1+\alpha)^r} \| x_1 - x_2\|^r. 
\end{align}
Let us now calculate
\begin{align}
\label{qjnlelqkwjeqweq}
\frac{\partial \hone(z)}{\partial \alpha} = \frac{r(1+\alpha)^{r-1} \|x_1-x_2\|^r  }{ (1+\alpha)^{2r} } \left( \alpha^{r-1} - \lambda_2 \right).
\end{align}
 According to (\ref{qjnlelqkwjeqweq}),  the function $\hone(z)$ is decreasing for $\alpha^{r-1} - \lambda_2 < 0$ and increasing for $\alpha^{r-1} - \lambda_2 > 0$. The global minimum is thus achieved for $\smash{\alpha = \lambda_2^{\frac{1}{r-1}}}$. Substituting this optimum value for $\alpha$ to (\ref{qowieoiqweqw2}), we obtain the same expression for $\hone(z)$ as in the statement of the lemma. This concludes the proof. 
\end{proof}

According to Lemma \ref{basiclemma}, given one data point at $X_1 = x_1$, and the other at $X_2 = x_2$, the minimum cost can be achieved by using a cluster center at
$z = \frac{x_1 + \alpha x_2}{1+ \alpha}$. Then, given a hypothetically-infinite number of cluster centers, one can achieve the optimal performance by placing the centers at every possible location imaginable. On the other hand, given only finitely many centers, one has no choice but be content with choosing the center that is close to the optimal location $z$. In such a scenario, it makes sense to analyze the behavior of the function $\hone(u)$ near the optimal value 
\begin{align}
u = z+\epsilon, 
\end{align}
where $\epsilon$ is a small vector in magnitude. The motivation behind restricting $\epsilon$ to be small is to be able to invoke a high-resolution analysis. We have
\begin{align}
\hone(z+\epsilon) & =  \left\|x_1 -  \frac{x_1 + \alpha x_2}{1 + \alpha} - \epsilon\right\|^r + \lambda \left\|x_2 -  \frac{x_1 + \alpha x_2}{1 + \alpha} - \epsilon\right\|^r \\
\label{pqowepqowiepqw} & = \frac{1}{(1+\alpha)^r} \left(\alpha^r \left\|x_1-x_2 - \epsilon \frac{1+\alpha}{\alpha} \right\|^r + \lambda  \left\| x_1 - x_2 + \epsilon(1+\alpha) \right\|^r\right).
\end{align} 
Now, let $\nabla f(w)$ and $\nabla^2 f(w) $ denote the gradient and the Hessian of a multivariate function $f$ evaluated at $w$, respectively. We have the generic multivariate Taylor series expansion
\begin{align}
\label{qijwekqjwelkqwlekjqweq}
f(w+\epsilon) = f(w) + \epsilon^T \nabla f(w) + \frac{1}{2}\epsilon^T \nabla^2 f(w) \epsilon + o(\|\epsilon\|^2).
\end{align}
In order to expand (\ref{pqowepqowiepqw}) for small $\epsilon$, we need to find the Taylor expansion of the function $w \mapsto \|w\|^r$. For this purpose, for any vector $w$, we let $\vec{w} = w/\|w\|$. The gradient and the Hessian of $w\mapsto \|w\|^r$  can then be calculated as
\begin{align}
\label{gradwr}
\nabla \|w\|^r & = r \|w\|^{r-2} w, 
\end{align}
and
\begin{align}
\label{hesswr}
\nabla^2 \|w\|^r & =  r\|w\|^{r-2} \left(  \mathbf{I} +    (r-2) \vec{w}\vec{w}^T \right),
\end{align}
respectively, where $\mathbf{I}$ is the identity matrix and $\vec{w} \triangleq w / \|w\|$. Substituting (\ref{gradwr}) and (\ref{hesswr}) to (\ref{qijwekqjwelkqwlekjqweq}), we obtain
\begin{align}
\label{oqwueoqpwieqwe}
\|w + \epsilon\|^r = \|w\|^r + r \|w\|^{r-2} w^T \epsilon + \frac{r\|w\|^{r-2}}{2} \epsilon^T \left(  \mathbf{I} +    (r-2) \vec{w}\vec{w}^T \right) \epsilon + o(\|\epsilon\|^2).
\end{align}
Using this expansion in (\ref{pqowepqowiepqw}) leads to
\begin{align}
\hone(z+\epsilon) = \frac{\alpha^{r-1}}{(1+\alpha)^{r-1}} \|v\|^r +   \frac{\alpha^{r-2}}{2(1+\alpha)^{r-3}}  r\|v\|^{r-2} \epsilon^T \left(  \mathbf{I} +    (r-2) \vec{v}\vec{v}^T \right) \epsilon + o(\|\epsilon\|^2),
\end{align}
where $v\triangleq x - y$. Using the equivalence $u = z+ \epsilon$, we have
\begin{align}
\label{oqiweuoqiwueqweqwe}
\hone(u) = \frac{\alpha^{r-1}\|v\|^r}{(1+\alpha)^{r-1}}  + \frac{\alpha^{r-2}r\|v\|^{r-2}}{2(1+\alpha)^{r-3}}   (u-z)^T \left(  \mathbf{I} +    (r-2) \vec{v}\vec{v}^T \right) (u-z) + o(\|u-z\|^2).
\end{align}
Let $c_3 =   \frac{r\alpha^{r-2}}{2(1+\alpha)^{r-3}}$ and $c_4 =  (\frac{\alpha}{1+\alpha})^{r-1}  E\|X_1-X_2\|^r$. Substituting (\ref{oqiweuoqiwueqweqwe}) to (\ref{qpiwoeiqpwoeiqw}), and using the fact that $\min_i(\phi_i + o(\phi_i)) = \min_i \phi_i + o(\min_i \phi_i)$for arbitrary functions $\phi_i$, yields
\begin{multline}
 \delta(U) = c_4 + c_3 \iint  \min_i \|v\|^{r-2}   (u_i-z)^T \left(  \mathbf{I} +   (r-2) \vec{v}\vec{v}^T \right) (u_i -z)f_{X_1,X_2}(x_1,x_2)\mathrm{d}x_1\mathrm{d}x_2 + \\ \iint o\left(\min_i \|u_i - z\|^2 \right) f_{X_1,X_2}(x_1,x_2)\mathrm{d}x_1\mathrm{d}x_2.
 \end{multline}
According to (\ref{asdasdasdasd}), the last term decays as $o(n^{-\frac{2}{d}})$. For the second term, we apply a change of variables $v = x - y$ and $z  = \frac{x+\alpha y}{1+\alpha}$ to obtain
\begin{align}
\label{kabsdknabqwheqwe}
\delta(U) =  c_4  + c_3 \int \min_i (z-u_i)^T B(z)  (z-u_i) \mathrm{d}z + o(n^{-\frac{2}{d}}),
\end{align}
where
\begin{align} 
\label{bzdefinition}
B(z) = \int \|v\|^{r-2} \left(  \mathbf{I} +    (r-2) \vec{v}\vec{v}^T \right)   f_{X_1,X_2}\left(z+\frac{\alpha v}{1+\alpha},z-\frac{v}{1+\alpha}\right)\mathrm{d} v.
\end{align}
Note that for any $v$, and $r \geq 1$, the matrix $\mathbf{I} +    (r-2) \vec{v}\vec{v}^T$ is positive semi-definite. This implies that the matrix $B(z)$ is positive semi-definite for every $r \geq 1$. 

The function $(z,u) \mapsto (z-u)^T B(z)  (z-u)$ defines an input-weighted quadratic distortion measure. The structure and distortion of the optimal quantizers corresponding to such distortion measures has been studied in \cite{li1999asymptotic}. As discussed above, the matrix $B(z)$ is positive semi-definite for every $z$ so that the results of \cite{li1999asymptotic} is applicable. In particular, we have
\begin{align}
\label{qpoeipqwoieqw123123}
\min_U \int \min_i (z-u_i)^T B(z)  (z-u_i) \mathrm{d}z = \kappa_{d} n^{-\frac{2}{d}} \left\|(\mathrm{det}B(z))^{\frac{1}{d}}\right\|_{\frac{d}{d+2}}  + o(n^{-\frac{2}{d}}),
\end{align}
and the optimal point density that achieves the asymptotic performance in (\ref{qpoeipqwoieqw123123}) is
\begin{align}
\label{nsdmans99}
z \mapsto \frac{(\mathrm{det}B(z))^{\frac{1}{d+2}}}{\int (\mathrm{det}B(z))^{\frac{1}{d+2}} \mathrm{d}z}.
\end{align}
Applying these results to our specific problem leads to the following theorem. 

\begin{theorem}
\label{theorem2}
Let $r \geq 1$ and $L=2$. As $n\rightarrow\infty$, we have 
\begin{align}
\label{qlwejqlkwejqwe}
\min_U \delta(U) = c_4 + c_3 \kappa_{d} n^{-\frac{2}{d}} \left\|(\mathrm{det}B(z))^{\frac{1}{d}}\right\|_{\frac{d}{d+2}} + o(n^{-\frac{2}{d}}),
\end{align}
where $B(z)$ is defined in (\ref{bzdefinition}). The optimal point density function that achieves (\ref{qlwejqlkwejqwe}) is proportional to $\smash{(\mathrm{det}B(z))^{\frac{1}{d+2}}}$.
\end{theorem}
\begin{proof}
The asymptotic distortion in (\ref{qlwejqlkwejqwe})  follows immediately from (\ref{kabsdknabqwheqwe}) and (\ref{qpoeipqwoieqw123123}). The optimal point density function is as given by (\ref{nsdmans99}).
\end{proof}

For ordinary center-based quantization with the $r$th power distortion measure, the average distortion decays as $n^{-\frac{r}{d}}$, see \cite{1056490}. It is interesting to note that, when one instead consider a sum of $r$th powers of distortions, as in \ref{qpiwoeiqpwoeiqw}, the average distortion decays as $n^{-\frac{2}{d}}$, independently of $r$.

Let us now discuss certain special cases of the conclusions of Theorem \ref{theorem2} above. 
\begin{example}
We first compare the results of Theorems \ref{theorem1} and \ref{theorem2} for $r=2$. Note that when $r=2$, we have $\alpha = \lambda_2$, and we can easily verify that $c_4 = c_3$ and $c_2 = c_1$. Moreover,
\begin{align}
B(z) & =  \int  f_{X_1,X_2}\left(z+\frac{\alpha v}{1+\alpha},z-\frac{v}{1+\alpha}\right)\mathrm{d} v \mathbf{I}  \\
& = \frac{1+\alpha}{\alpha} \int f_{X_1,X_2}\left(v,\frac{z(1+\alpha)-v}{\alpha}\right)\mathrm{d}v \mathbf{I}  \\
& = f_Z(z) \mathbf{I}
\end{align}
The second equality follows from a change of variables $z+\frac{\alpha v}{1+\alpha} \leftarrow v$, and the last equality follows once we view the PDF of $\smash{Z = \frac{X + \alpha X_2}{1+\alpha}}$ as a convolution. Substituting the derived equalities, the conclusions of Theorems \ref{theorem1} and \ref{theorem2} become identical whenever $r=2$.\hfill\qedsymbol
\end{example}
\begin{example}
\label{example2}
Let $d=1$, and suppose $X_1$ and $X_2$ are independent and uniform random variables on $[0,1]$. In order to calculate $B(z)$, we find the region where the joint PDF of $X_1$ and $X_2$ in (\ref{bzdefinition}) is non-zero. In other words, we solve for $v$ in the conditions $0 \leq z+\frac{\alpha v}{1+\alpha} \leq 1$ and $0\leq z-\frac{v}{1+\alpha} \leq 1$. After some straightforward manipulations, we obtain the equivalent set of inequalities
\begin{align}
\label{condzzz1} -z\frac{1+\alpha}{\alpha} \leq v \leq (1-z)\frac{1+\alpha}{\alpha}, \\
\label{condzzz2} -(1-z)(1+\alpha) \leq v \leq z(1+\alpha).
\end{align}

\begin{figure}[h]
\begin{center}
\scalebox{0.65}{\includegraphics{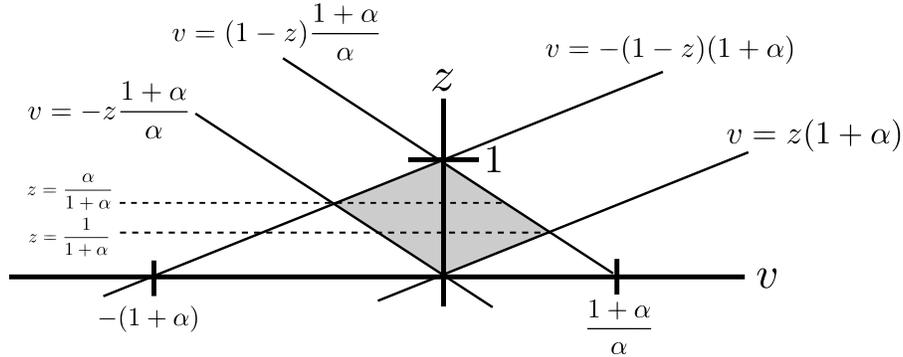}}
\end{center}
\caption{Calculation of the optimal quantizer point density.}
\label{uniformcalcfig}
\end{figure}

For $\alpha > 1$, the geometry of the region defined by the inequalities in (\ref{condzzz1}) and (\ref{condzzz2}) is illustrated in Fig. \ref{uniformcalcfig} as a shaded quadrangle. Correspondingly, we obtain
\begin{align}
\label{bzerpssposdasdA}
\! \! \!  B(z) \!=\! \left\{\! \!  \begin{array}{rl} \int_{-z\frac{1+\alpha}{\alpha}}^{z(1+\alpha)} |v|^{r-2}(r\!-\!1)\mathrm{d}v = z^{r-1}(1+\alpha)^{r-1}(1\!+\!\frac{1}{\alpha^{r-1}}), \!\! & z\!\in\![0,\frac{1}{1+\alpha}], \\ 
\int_{-z\frac{1+\alpha}{\alpha}}^{(1-z)\frac{1+\alpha}{\alpha}} |v|^{r-2}(r\!-\!1)\mathrm{d}v = (\frac{1+\alpha}{\alpha})^{r-1}((1\!-\!z)^{r-1}\!+\!z^{r-1}),\!\!  & z\!\in\![\frac{1}{1+\alpha} , \frac{\alpha}{1+\alpha}], \\ 
\int_{-(1-z)(1+\alpha)}^{(1-z)\frac{1+\alpha}{\alpha}} |v|^{r-2}(r\!-\!1)\mathrm{d}v = (1\!-\!z)^{r-1}(1\!+\!\alpha)^{r-1}(1\!+\!\frac{1}{\alpha^{r-1}}),\!\!  & z\!\in\![\frac{\alpha}{1+\alpha} ,1], \\
0,\!\!  & z\!\notin\![0,1]. 
 \end{array} \right.\! \!
\end{align}
According to Theorem \ref{theorem2}, the optimal point density at $z$ is proportional to the cube root of $B(z)$. The normalizing constant can be calculated to be
\begin{align}
\! \!\!  \int_0^1\! \left(B(z)\right)^{\frac{1}{3}}\!\mathrm{d}z 
 \label{laksjdlaksjda} & = \frac{6}{(r\!+\!2)(1\!+\!\alpha)}\Bigl(1\!+\!\frac{1}{\alpha^{r-1}}\Bigr)^{\!\frac{1}{3}} \!+\!
\Bigl(1\!+\!\frac{1}{\alpha}\Bigr)^{\!\frac{r-1}{3}}\!\!\!\int_{\frac{1}{1+\alpha}}^{\frac{\alpha}{1+\alpha}}((1\!-\!z)^{r-1}\!+\!z^{r-1})^{\frac{1}{3}}\mathrm{d}z. \!\!
\end{align}
The integral in (\ref{laksjdlaksjda}) cannot be expressed in terms of elementary functions, but can easily be evaluated numerically. Also, for the special case of $\alpha = 1$, the integral vanishes so that we  have, simply
\begin{align}
\label{pqiwjepoqjwepoqjwe1}
\int_0^1 \left(B(z)\right)^{\frac{1}{3}}\mathrm{d}z   = \frac{2^{\frac{1}{3}}3}{r+2}.
\end{align}
Also, when $X_1$ and $X_2$ are independent and uniform on $[0,1]$, the random variable $\|X_1 - X_2\| = |X_1 - X_2|$ has PDF $f_{|X_1 - X_2|}(z) = 2(1-z),\,z\in[0,1]$. Therefore,
\begin{align}
\label{pqiwjepoqjwepoqjwe2}
\mathrm{E}\|X_1 - X_2\|^r = \int_0^1 z^r f_{|X_1-X_2|}(z) \mathrm{d}z = \frac{2}{(r+1)(r+2)}.
\end{align}
A closed-form  asymptotic expression for the optimal asymptotic distortion can then obtained by substituting (\ref{laksjdlaksjda}) and (\ref{pqiwjepoqjwepoqjwe2}) to Theorem \ref{theorem2}. One only needs to numerically evaluate the integral in (\ref{laksjdlaksjda}). In particular, for $\alpha = 1$, i.e. when the samples from $X_1$ and $X_2$ are weighted equally, we obtain
\begin{align}
\delta(U) =  \frac{2^r}{(r+1)(r+2)}  + \frac{18 r}{2^{r}(r+2)^3} \frac{1}{ n^2}+ o\Bigl(\frac{1}{ n^2}\Bigr),
\end{align}
According to (\ref{bzerpssposdasdA}) and Theorem \ref{theorem2}, the optimal point density function that achieves (\ref{kabsdknabqwheqwe}) is proportional to $(1 - |2z - 1|)^{r-1}$ on $[0,1]$, and vanishes everywhere else. 

The case $\alpha  < 1$ can be handled in a similar manner. This concludes our example.
\hfill\qedsymbol
\end{example}

It is not clear how to extend the analysis to $L > 2$ sources. Note however that the numerical design of the quantizer (i.e. a minimization of (\ref{objfun})) in the general case is always possible via the following variant of the generalized Lloyd algorithm: First, one initializes some arbitrary $U$, and then iterates between the two steps of calculating the generalized Voronoi cells
\begin{align}
V_j \leftarrow \left\{[x_1 \cdots x_L]: \sum_{i = 1}^L \lambda_{i} \|u_j - x_{i}\|^r \leq \sum_{i = 1}^L \lambda_{i} \|u_k - x_{i}\|^r,\,\forall k\right\},\,j=1,\ldots,n,
\end{align}
and the generalized centers
\begin{align}
\label{qpoweipqoweqw}
u_j \leftarrow \arg\min_u \int_{V_j} \sum_{i = 1}^L \lambda_{i} \|u - x_{i}\|^r f_{\mathbf{X}}(\mathbf{x})\mathrm{d}\mathbf{x}.
\end{align}
It is easily seen that this algorithm results in a non-increasing average distortion at every iteration and thus it converges in a cost-function sense.  Moreover, the center calculation (\ref{qpoweipqoweqw}) can be accomplished in a computationally-efficient manner as it is convex for any $r \geq 1$. We will use this algorithm to validate our analytical results in the next section over various datasets.

\section{Numerical Results}
\label{secNumerical}

In this section, we provide numerical experiments to show the performance of our algorithms over various datasets and to verify our analytical results. 

\subsection{Performance Results for Clustering Through Noisy Observations}
First, we show that our proposed approach of clustering noisy observation vectors to a common center is superior to naively clustering a concatenated version of noisy observation vectors. We follow the same notation as in Section \ref{secAppScenarios}, but consider now a real dataset with a practical noisy observation scenario. Specifically, let $\mathcal{D}'$ be the Iris dataset \cite{fisher1936use}, consisting of $150$ four dimensional vectors, where the components of the vectors correspond to the sepal length, sepal width, petal length, and petal width of each flower. In practice, one may not be able to access the true $\mathcal{D}'$, but its noisy version: For example, multiple drones may measure a given iris from a distance over the air, resulting in multiple noisy observations of a sample in $\mathcal{D}'$. Mathematically, let $X'$ be uniformly distributed on $\mathcal{D}'$. We assume that the observations are given by the random variables $X_i = X' + N_i,\,i=1,\ldots,L$, where $N_1,\ldots,N_L$ are independent noise random variables. Through the samples $X_1,\ldots,X_L$, our goal is to obtain a clustering that is as close to the clustering of $\mathcal{D}'$ as possible. 

\begin{figure}
\begin{center}
\scalebox{0.6}{\includegraphics{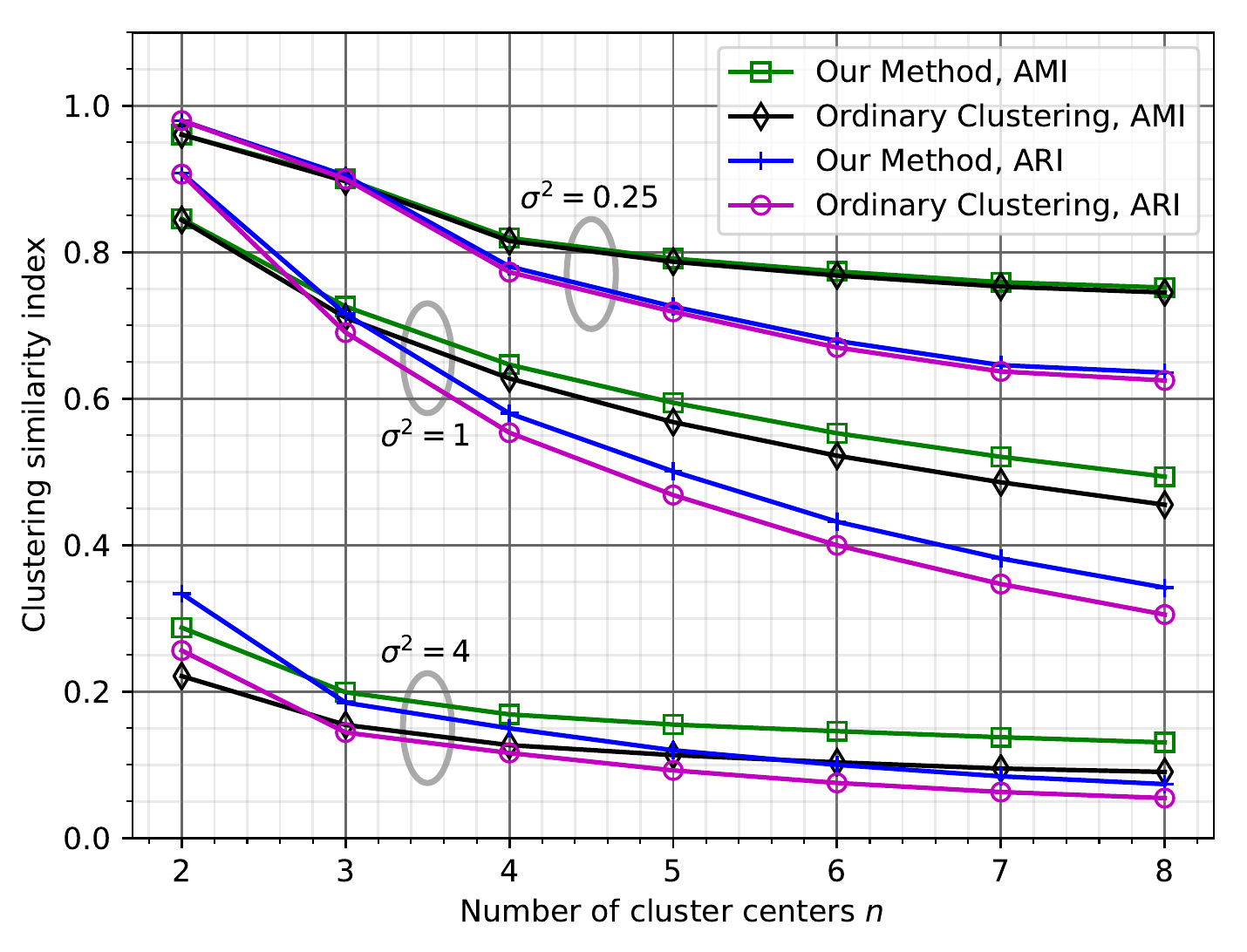}}
\end{center}
\caption{Clustering via four observations with Gaussian noise.}
\label{fig1}\vspace{-5pt}
\end{figure}

One one hand, one can consider the ordinary $k$-means clustering of $[X_1\cdots X_L]$, which we refer to as the ``ordinary clustering'' scenario. Our method instead relies on minimizing (\ref{ourproblem}) for the special case $r=2$. 
 In Fig. \ref{fig1}, we compare the clusterings obtained using two approaches in terms of their similarity with the $k$-means clustering of the original dataset $\mathcal{D}'$. The horizontal axis is the number of cluster centers, and the vertical axis represents the similarity measure. We consider both the ARI and AMI similarity measures. A Monte Carlo similarity measure average $\gamma$ is  accurate within $[\gamma-0.001,\gamma+0.001]$ with $95\%$ confidence. We also consider four observations $L=4$, and that $N_i$ are zero-mean Gaussian random variables with variance $\sigma^2$. First, we can observe that our approach outperforms ordinary clustering uniformly for all scenarios and for both similarity measures. The improvement is particularly notable when there is more noise in the observations or when the number of cluster centers are large. In particular, for $3$ cluster centers, which is the ground truth for the number of clusters of the dataset, the improvement of our method over ordinary clustering in terms of the ARI measure is $0.5\%$, $3.7\%$, and $28\%$ for $\sigma^2 = 0.25,1,$ and $4$, respectively. 
 
 \begin{figure}
 \begin{center}
\scalebox{0.6}{\includegraphics{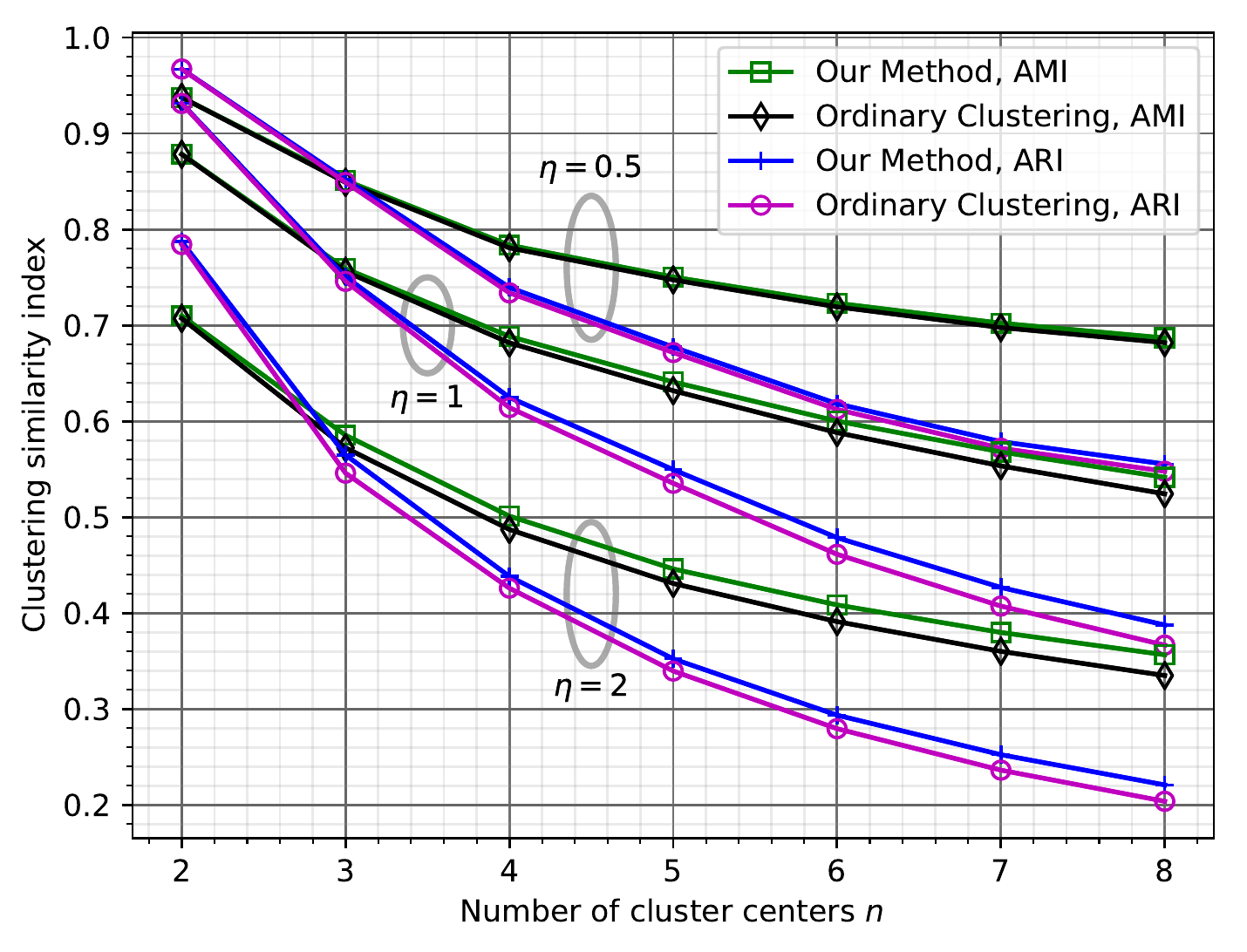}}
\end{center}
\caption{Clustering via two observations with uniform noise.}
\label{fig2}
\end{figure}
 
 Advantage of our method over ordinary clustering carries well over non-Gaussian noise models and different number of observations over different datasets. In Fig. \ref{fig2},  we show the results for the case where each $N_i$ are uniformly distributed on $[-\eta,\eta]$ for $\eta\in\{0.5,1,2\}$ for $L=2$. We can similarly observe that our approach that relies on quantizing to a common cluster center provides a better performance.
 
In Fig. \ref{fig2xxxx}, we show the results for UCI Wine dataset \cite{vandeginste1990parvus}, which consists of $178$ $13$-dimensional vectors. Since, for this dataset, the components of vectors have vastly different variances, we have preprocessed the dataset such that each component has unit variance. Except for the case of $2$ cluster centers, our quantization approach outperforms the ordinary $k$-means algorithm for all scenarios. 

 \begin{figure}
 \begin{center}
\scalebox{0.6}{\includegraphics{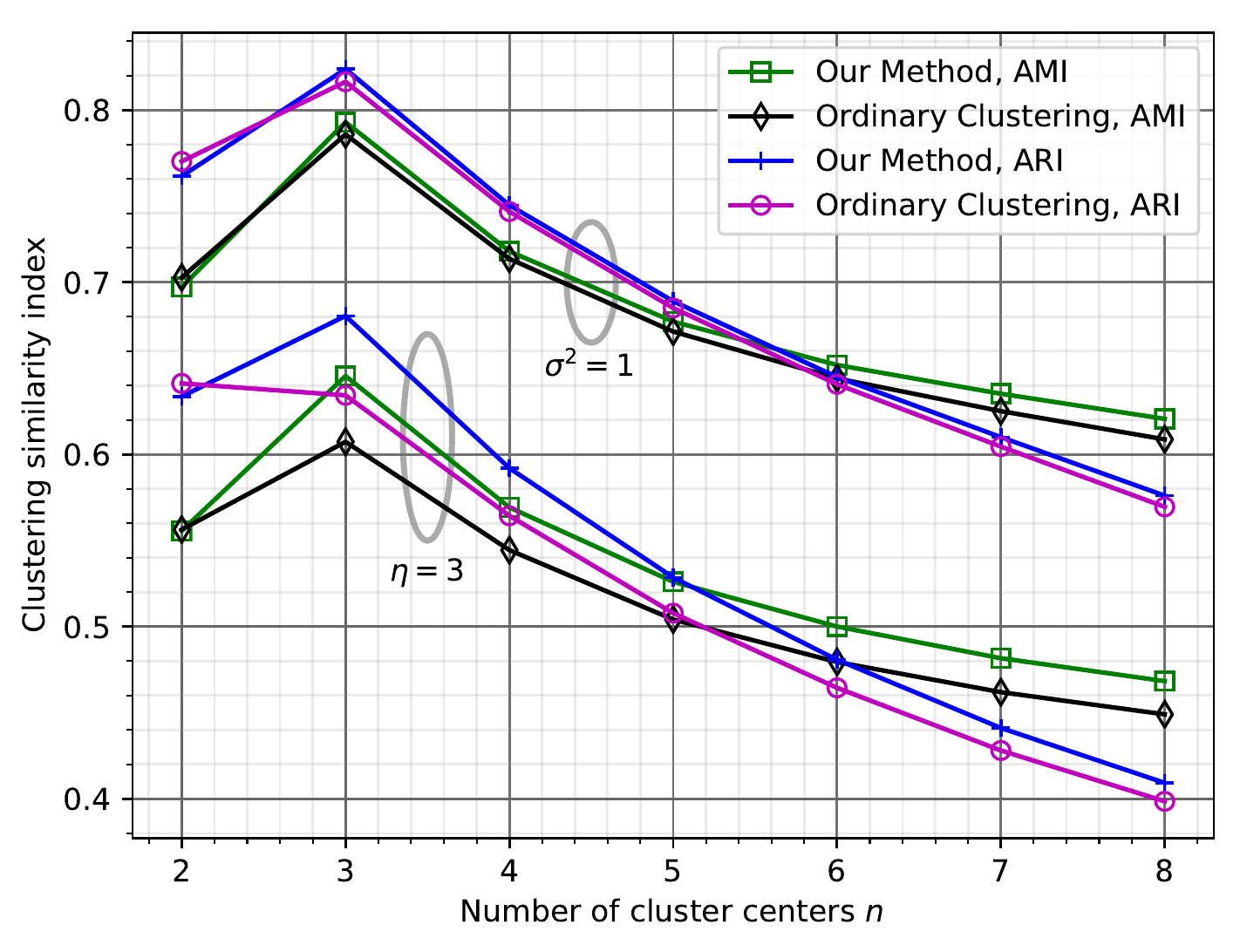}}
\end{center}
\caption{Clustering via noisy observations from the Wine dataset.}
\label{fig2xxxx}
\end{figure}

%
%

\begin{figure}
    \centering
    \begin{subfigure}[b]{0.49\textwidth}
    \scalebox{0.45}{\includegraphics{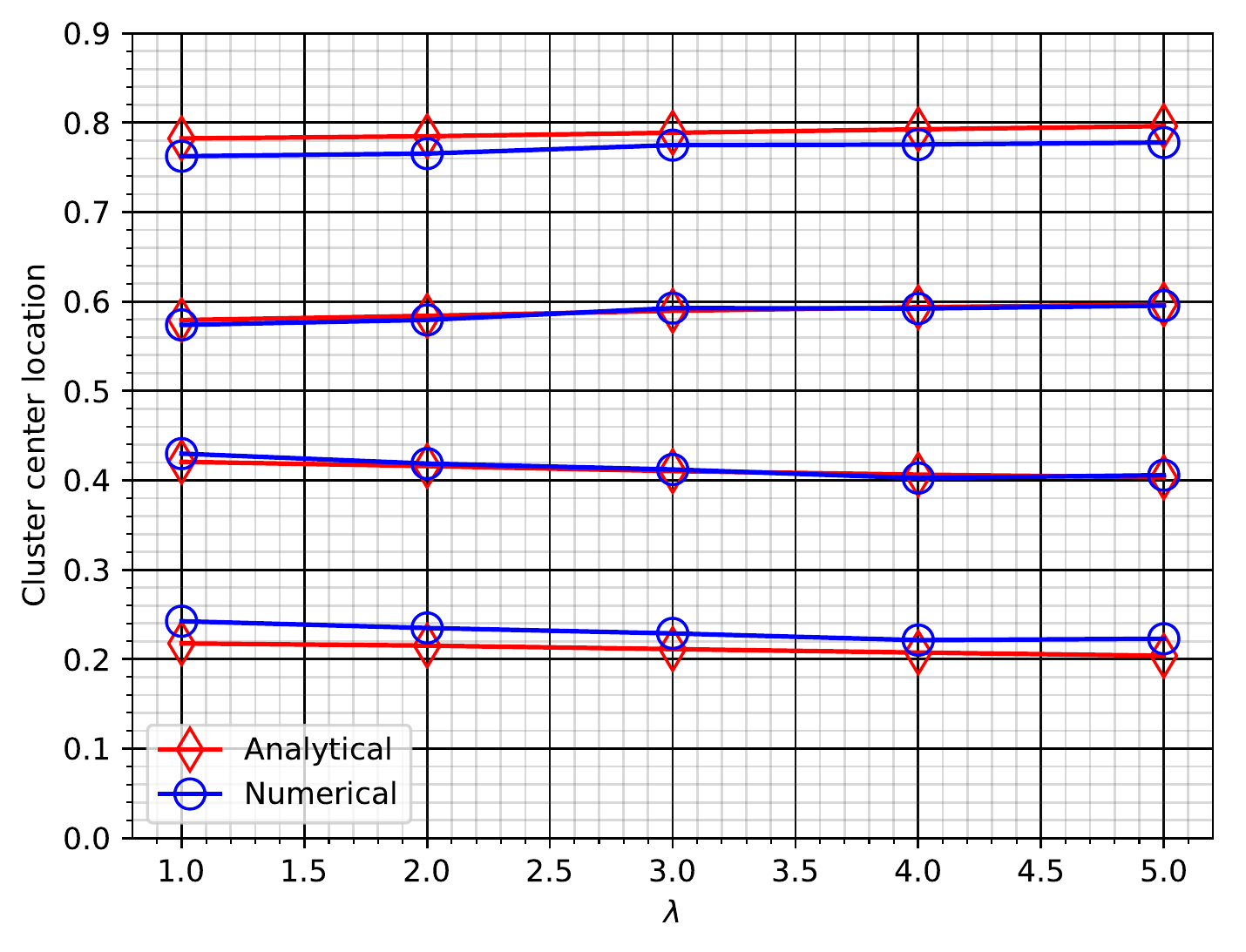}}
        \caption{$n=4,r=3$}
                \label{figlocs1}
    \end{subfigure}
        \begin{subfigure}[b]{0.49\textwidth}
    \scalebox{0.45}{\includegraphics{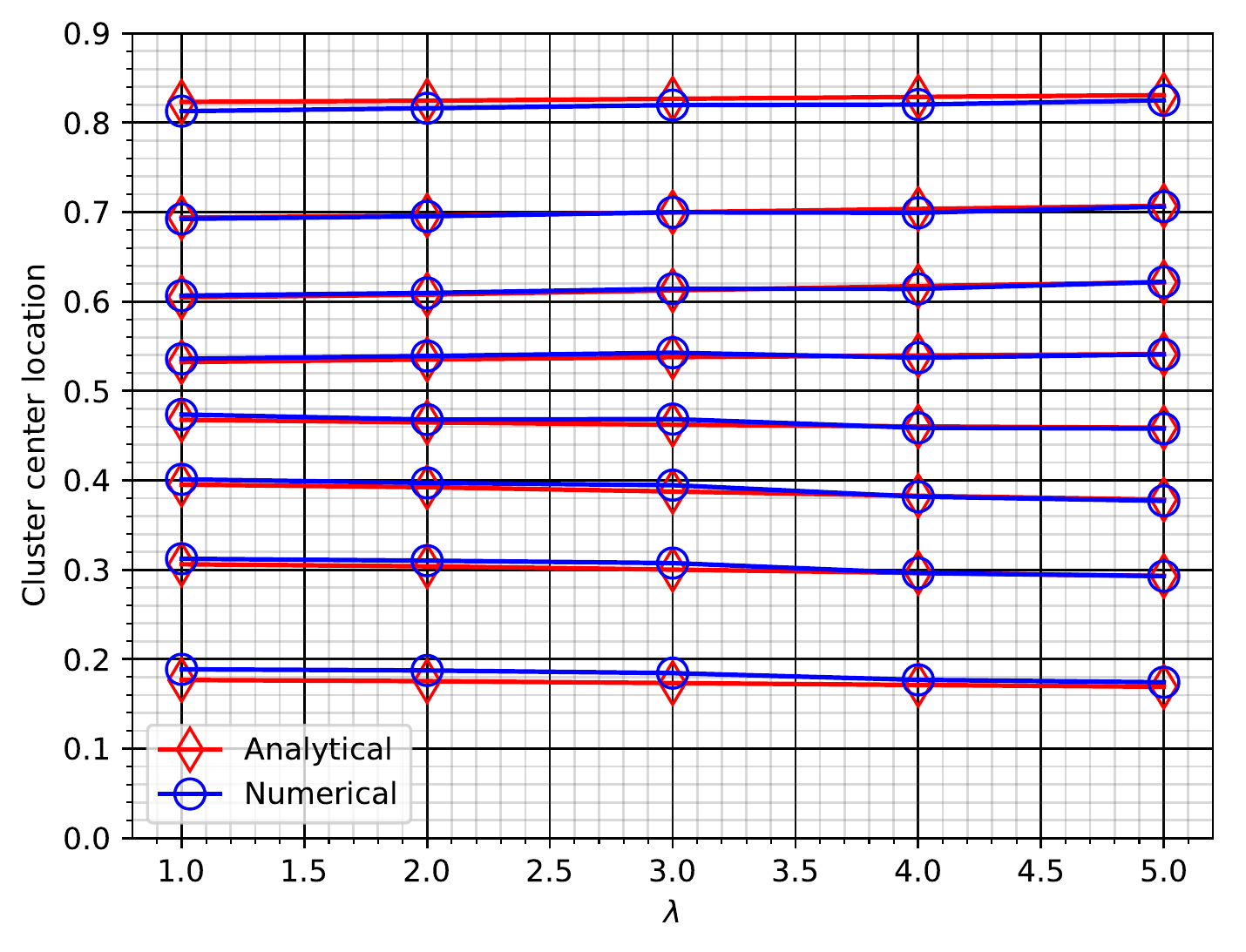}}
    \caption{$n=8,r=4$}
        \label{figlocs2}
    \end{subfigure}
    \caption{Analytical vs. numerical cluster center locations}
    \label{figlocsgeneral}
\end{figure}

\subsection{Validation of High Resolution Analysis}
We now provide numerical experiments that verify the high resolution results provided by Theorems \ref{theorem1} and \ref{theorem2}. We consider the same scenario as in Example \ref{example2} for different values of $r$ and $n$. In Fig. \ref{figlocsgeneral}, we compare the cluster centers obtained using our generalized Lloyd algorithm (labeled ``Numerical'') with those provided by Theorem \ref{theorem2} (labeled ``Analytical'') for the cases $n=4,\,r=3$ in Fig. \ref{figlocs1} and $n=8,\,r=4$ in Fig. \ref{figlocs2}. The horizontal axis represents $\lambda\in\{1,\ldots,5\}$, and the vertical axis represents the cluster center or the quantization point locations. Note that Theorem \ref{theorem2} provides the optimal quantizer point density function, not the individual quantization points or cluster centers. We may use, however, inverse transform sampling to obtain a sequence of quantization points that will be faithful to the quantizer point density function. Namely, if the desired point density function is $\lambda(x)$, we use the quantization points $\smash{\Lambda^{-1}(\frac{2i - 1}{2n}),\,i=1,\ldots,n}$,  where $\smash{\Lambda(y) \triangleq \int_{-\infty}^y \lambda(x)\mathrm{d}x}$ is the cumulative point density function. \begin{wrapfigure}{r}{0.49\textwidth} 
\scalebox{0.45}{\includegraphics{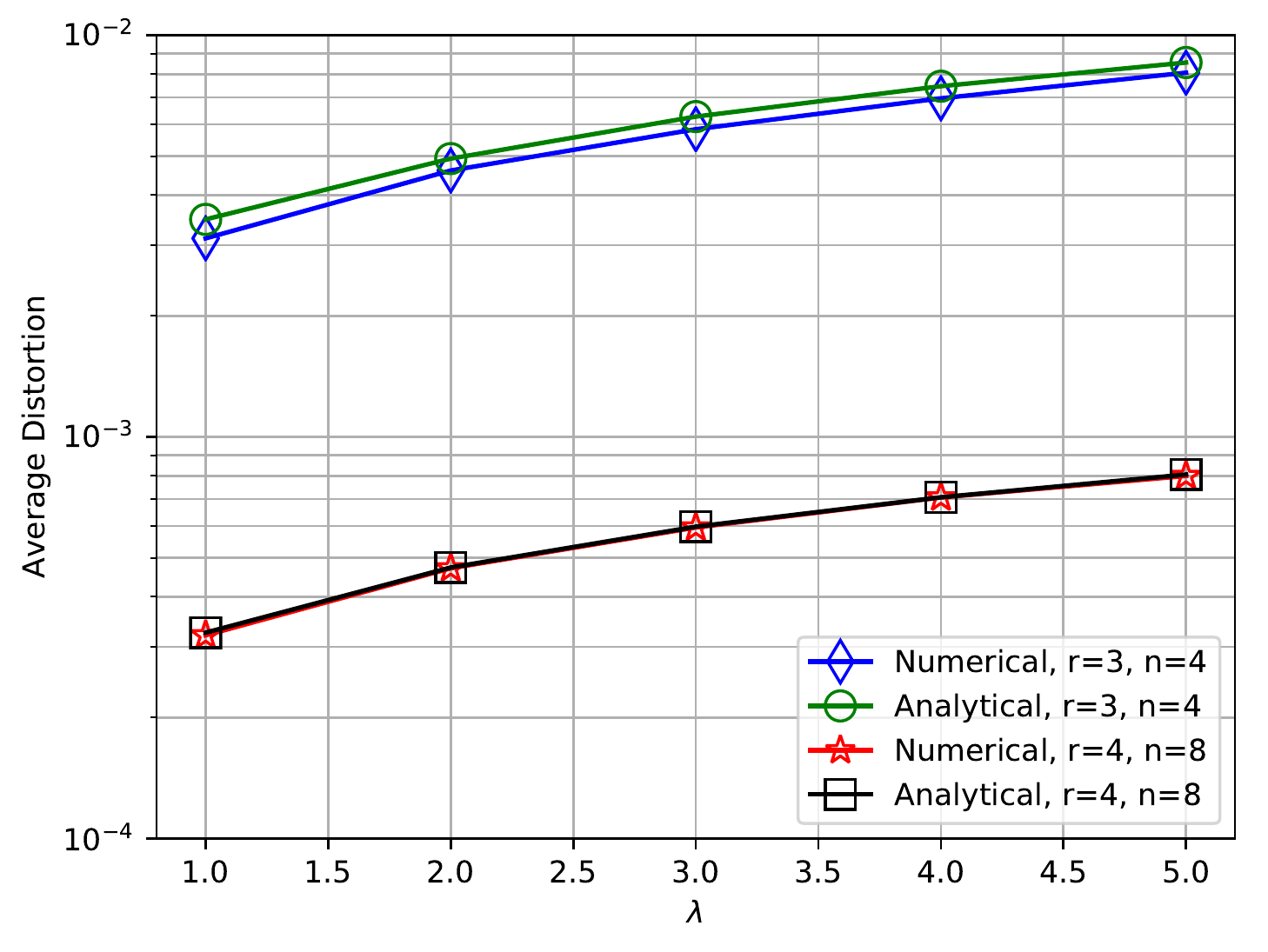}}
\caption{Analytical vs. numerical average distortions.}
\label{figavdistos}
\end{wrapfigure}Although the results of Theorem \ref{theorem2} are only valid asymptotically, we can observe that they still provide a very accurate description of the optimal quantizers for number of cluster centers as low as $n=4$. There is only slight mismatch for centers that are close to the boundaries $0$ and $1$. A similar phenomena can be observed for the case $n=8$ and $r=4$ in Fig. \ref{figlocs2}, but the amount of mismatch is lower. Also, the theory can precisely predict very subtle changes in the optimal quantization points, e.g. the movement of the optimal location for the third quantization point from $0.4$ from $0.38$ as $\lambda$ grows from $1$ to $5$. In Fig. \ref{figavdistos}, we show the average distortion performances corresponding to the cluster centers in Fig. \ref{figlocsgeneral}. Again, quantization theory can precisely predict the average distortion performance even when $n$ is as small as $4$. The difference between the analytical and the numerical results become indistinguishable for the case of $n=8$ cluster centers. 

For $r=2$, the proposed algorithms correspond to a simple scheme where one quantizes a weighted average of the noisy observations. The case $r < 2$ is especially useful for handling datasets with outlier observations (e.g. many noiseless observations and few noisy observations). Numerical simulations corresponding to these cases will be provided elsewhere.

\section{Conclusion}
\label{secConclusions}
We have considered quantizing an $Ld$-dimensional sample, which is obtained by concatenating $L$ vectors from datasets of $d$-dimensional vectors, to a $d$-dimensional cluster center. An application area of such a quantization strategy is when one wishes to cluster a dataset through $L$ noisy observations of each of its members. We have found analytical formulae for the average distortion performance and the optimal quantizer structure in the asymptotic regime where the number of cluster centers are large. We have shown that our clustering approach outperforms the naive approach that relies on quantizing the $Ld$-dimensional noisy observation vectors to $Ld$-dimensional centers.

\bibliographystyle{IEEEtran}
\bibliography{references} 

\begin{thebibliography}{10}
\providecommand{\url}[1]{#1}
\csname url@samestyle\endcsname
\providecommand{\newblock}{\relax}
\providecommand{\bibinfo}[2]{#2}
\providecommand{\BIBentrySTDinterwordspacing}{\spaceskip=0pt\relax}
\providecommand{\BIBentryALTinterwordstretchfactor}{4}
\providecommand{\BIBentryALTinterwordspacing}{\spaceskip=\fontdimen2\font plus
\BIBentryALTinterwordstretchfactor\fontdimen3\font minus
  \fontdimen4\font\relax}
\providecommand{\BIBforeignlanguage}[2]{{%
\expandafter\ifx\csname l@#1\endcsname\relax
\typeout{** WARNING: IEEEtran.bst: No hyphenation pattern has been}%
\typeout{** loaded for the language `#1'. Using the pattern for}%
\typeout{** the default language instead.}%
\else
\language=\csname l@#1\endcsname
\fi
#2}}
\providecommand{\BIBdecl}{\relax}
\BIBdecl

\bibitem{jain1999data}
A.~K. Jain, M.~N. Murty, and P.~J. Flynn, ``Data clustering: a review,''
  \emph{ACM computing surveys (CSUR)}, vol.~31, no.~3, pp. 264--323, 1999.

\bibitem{jain2010data}
A.~K. Jain, ``Data clustering: 50 years beyond k-means,'' \emph{Pattern
  recognition letters}, vol.~31, no.~8, pp. 651--666, 2010.

\bibitem{gan2007data}
G.~Gan, C.~Ma, and J.~Wu, \emph{Data clustering: theory, algorithms, and
  applications}.\hskip 1em plus 0.5em minus 0.4em\relax Siam, 2007, vol.~20.

\bibitem{macqueen1967some}
J.~MacQueen, ``Some methods for classification and analysis of multivariate
  observations,'' in \emph{Proceedings of the Fifth Berkeley Symposium on
  Mathematical Statistics and Probability}, vol.~1, no.~14.\hskip 1em plus
  0.5em minus 0.4em\relax Oakland, CA, USA, 1967, pp. 281--297.

\bibitem{steinhauskmeans}
H.~Steinhaus, ``Sur la division des corps mat\'{e}riels en parties,''
  \emph{Bulletin de l'Acad\'{e}mie Polonaise des Sciences, Classe III},
  vol.~IV, no.~12, pp. 801--804, 1956.

\bibitem{lloyd1982least}
S.~Lloyd, ``Least squares quantization in {PCM},'' \emph{IEEE Transactions on
  Information Theory}, vol.~28, no.~2, pp. 129--137, 1982.

\bibitem{bucklew1982multidimensional}
J.~Bucklew and G.~Wise, ``Multidimensional asymptotic quantization theory with
  r th power distortion measures,'' \emph{IEEE Transactions on Information
  Theory}, vol.~28, no.~2, pp. 239--247, 1982.

\bibitem{466658}
W.~R. {Gardner} and B.~D. {Rao}, ``Theoretical analysis of the high-rate vector
  quantization of lpc parameters,'' \emph{IEEE Transactions on Speech and Audio
  Processing}, vol.~3, no.~5, pp. 367--381, Sep. 1995.

\bibitem{li1999asymptotic}
J.~Li, N.~Chaddha, and R.~M. Gray, ``Asymptotic performance of vector
  quantizers with a perceptual distortion measure,'' \emph{IEEE Transactions on
  Information Theory}, vol.~45, no.~4, pp. 1082--1091, Apr. 1999.

\bibitem{banerjee2005clustering}
A.~Banerjee, S.~Merugu, I.~S. Dhillon, and J.~Ghosh, ``Clustering with
  {B}regman divergences,'' \emph{Journal of Machine Learning Research}, vol.~6,
  no. Oct, pp. 1705--1749, 2005.

\bibitem{fischer2010quantization}
A.~Fischer, ``Quantization and clustering with {B}regman divergences,''
  \emph{Journal of Multivariate Analysis}, vol. 101, no.~9, pp. 2207--2221,
  2010.

\bibitem{aloise2009np}
D.~Aloise, A.~Deshpande, P.~Hansen, and P.~Popat, ``Np-hardness of euclidean
  sum-of-squares clustering,'' \emph{Machine learning}, vol.~75, no.~2, pp.
  245--248, 2009.

\bibitem{linde1980algorithm}
Y.~Linde, A.~Buzo, and R.~Gray, ``An algorithm for vector quantizer design,''
  \emph{IEEE Transactions on Communications}, vol.~28, no.~1, pp. 84--95, 1980.

\bibitem{720541}
R.~M. Gray and D.~L. Neuhoff, ``Quantization,'' \emph{IEEE Transactions on
  Information Theory}, vol.~44, no.~6, pp. 2325--2383, Oct 1998.

\bibitem{gersho2012vector}
A.~Gersho and R.~M. Gray, \emph{Vector quantization and signal
  compression}.\hskip 1em plus 0.5em minus 0.4em\relax Springer Science \&
  Business Media, 2012, vol. 159.

\bibitem{1056490}
P.~Zador, ``Asymptotic quantization error of continuous signals and the
  quantization dimension,'' \emph{IEEE Transactions on Information Theory},
  vol.~28, no.~2, pp. 139--149, March 1982.

\bibitem{liu2016clustering}
C.~Liu and M.~Belkin, ``Clustering with {B}regman divergences: An asymptotic
  analysis,'' in \emph{Advances in Neural Information Processing Systems},
  2016, pp. 2351--2359.

\bibitem{chao2017survey}
G.~Chao, S.~Sun, and J.~Bi, ``A survey on multi-view clustering,'' \emph{arXiv
  preprint arXiv:1712.06246}, 2017.

\bibitem{bickel2004multi}
S.~Bickel and T.~Scheffer, ``Multi-view clustering,'' in \emph{Fourth IEEE
  International Conference on Data Mining (ICDM'04)}.\hskip 1em plus 0.5em
  minus 0.4em\relax IEEE, 2004, pp. 19--26.

\bibitem{rand1971objective}
W.~M. Rand, ``Objective criteria for the evaluation of clustering methods,''
  \emph{Journal of the American Statistical Association}, vol.~66, no. 336, pp.
  846--850, 1971.

\bibitem{vinh2010information}
N.~X. Vinh, J.~Epps, and J.~Bailey, ``Information theoretic measures for
  clusterings comparison: Variants, properties, normalization and correction
  for chance,'' \emph{The Journal of Machine Learning Research}, vol.~11, pp.
  2837--2854, 2010.

\bibitem{farvardin1990study}
N.~Farvardin, ``A study of vector quantization for noisy channels,'' \emph{IEEE
  Transactions on Information Theory}, vol.~36, no.~4, pp. 799--809, 1990.

\bibitem{dave1991characterization}
R.~N. Dave, ``Characterization and detection of noise in clustering,''
  \emph{Pattern Recognition Letters}, vol.~12, no.~11, pp. 657--664, 1991.

\bibitem{okabe1997locational}
A.~Okabe and A.~Suzuki, ``Locational optimization problems solved through
  voronoi diagrams,'' \emph{European Journal of Operational Research}, vol.~98,
  no.~3, pp. 445--456, 1997.

\bibitem{farahani2010multiple}
R.~Z. Farahani, M.~SteadieSeifi, and N.~Asgari, ``Multiple criteria facility
  location problems: A survey,'' \emph{Applied Mathematical Modelling},
  vol.~34, no.~7, pp. 1689--1709, 2010.

\bibitem{meira2017clustering}
L.~A. Meira, F.~K. Miyazawa, and L.~L. Pedrosa, ``Clustering through continuous
  facility location problems,'' \emph{Theoretical Computer Science}, vol. 657,
  pp. 137--145, 2017.

\bibitem{meiraconfpapap}
L.~Meira and F.~Miyazawa, ``A continuous facility location problem and its
  application to a clustering problem,'' 01 2008, pp. 1826--1831.

\bibitem{czumaj20131}
A.~Czumaj, C.~Lammersen, M.~Monemizadeh, and C.~Sohler, ``(1+
  $\varepsilon$)-approximation for facility location in data streams,'' in
  \emph{Proceedings of the twenty-fourth annual ACM-SIAM symposium on Discrete
  algorithms}.\hskip 1em plus 0.5em minus 0.4em\relax SIAM, 2013, pp.
  1710--1728.

\bibitem{koyuncu2018power}
E.~Koyuncu, ``Power-efficient deployment of {UAV}s as relays,'' \emph{IEEE
  Signal Processing Advances in Wireless Communications (SPAWC)}, Jun. 2018.

\bibitem{ekc17}
J.~Guo, E.~Koyuncu, and H.~Jafarkhani, ``Energy efficiency in two-tiered
  wireless sensor networks,'' in \emph{IEEE International Conference on
  Communications (ICC)}, May 2017.

\bibitem{ekc18}
E.~Koyuncu, ``Performance gains of optimal antenna deployment for massive mimo
  systems,'' in \emph{IEEE Global Communications Conference (GLOBECOM)}, Dec.
  2017.

\bibitem{bennett1948spectra}
W.~R. Bennett, ``Spectra of quantized signals,'' \emph{The Bell System
  Technical Journal}, vol.~27, no.~3, pp. 446--472, 1948.

\bibitem{graf2007foundations}
S.~Graf and H.~Luschgy, \emph{Foundations of quantization for probability
  distributions}.\hskip 1em plus 0.5em minus 0.4em\relax Springer, 2007.

\bibitem{fisher1936use}
R.~A. Fisher, ``The use of multiple measurements in taxonomic problems,''
  \emph{Annals of eugenics}, vol.~7, no.~2, pp. 179--188, 1936.

\bibitem{vandeginste1990parvus}
B.~Vandeginste, ``Parvus: An extendable package of programs for data
  exploration, classification and correlation, m. forina, r. leardi, c.
  armanino and s. lanteri, elsevier, amsterdam, 1988,'' \emph{Journal of
  Chemometrics}, vol.~4, no.~2, pp. 191--193, 1990.

\end{thebibliography}

\end{document}